%% file: Paper.tex
\documentclass{article}
\usepackage{algorithm}
\usepackage[noend]{algorithmic}
\usepackage{amsmath}
\usepackage{amssymb}
\usepackage{amsthm}
\usepackage{bbm}
\usepackage{bm}
\usepackage{color}
\usepackage{dsfont}
\usepackage{enumerate}
\usepackage{enumitem}
\usepackage{graphicx}
\usepackage{hyperref}
\usepackage[accepted]{icml2017}
\usepackage{natbib}
\usepackage{subfigure}
\usepackage{tabularx}
\usepackage{times}
\usepackage[usenames,dvipsnames]{xcolor}

\usepackage[capitalize,noabbrev]{cleveref}

\newcommand{\commentout}[1]{}
\newcommand{\junk}[1]{}

\newtheorem{theorem}{Theorem}

\newtheorem{lemma}[theorem]{Lemma}

\newtheorem{assumption}{Assumption}

\newcommand{\cD}{\mathcal{D}}
\newcommand{\cE}{\mathcal{E}}
\newcommand{\ccE}{\overline{\cE}}
\newcommand{\cR}{\mathcal{R}}

\newcommand{\len}{\mathrm{len}}

\newcommand{\abs}[1]{\left|#1\right|}
\newcommand{\ceils}[1]{\left\lceil#1\right\rceil}
\newcommand{\condE}[2]{\mathbb{E} \left[#1 \,\middle|\, #2\right]}

\newcommand{\E}[1]{\mathbb{E} \left[#1\right]}

\newcommand{\I}[1]{\mathds{1} \! \left\{#1\right\}}

\newcommand{\kl}[2]{D_\mathrm{KL}\hspace{-1mm}\left(#1 \,\|\, #2\right)}

\newcommand{\rnd}[1]{\bm{#1}}
\newcommand{\set}[1]{\left\{#1\right\}}

\DeclareMathOperator*{\argmax}{arg\,max\,}
\DeclareMathOperator*{\argmin}{arg\,min\,}
\mathchardef\mhyphen="2D

\newcommand{\cascadeklucb}{{\tt CascadeKL\mhyphen UCB}}

\newcommand{\expthree}{{\tt Exp3}}

\newcommand{\klucb}{{\tt KL\mhyphen UCB}}

\newcommand{\mergerank}{{\tt BatchRank}}
\newcommand{\rankedexpthree}{{\tt RankedExp3}}

\newcommand{\rankoneelim}{{\tt Rank1Elim}}
\newcommand{\ucb}{{\tt UCB1}}

\newcommand{\attrprob}{\alpha}
\newcommand{\attrRV}{\rnd{A}}
\newcommand{\attrval}{A}
\newcommand{\examprob}{\chi}
\newcommand{\examRV}{\rnd{X}}
\newcommand{\examval}{X}
\newcommand{\avgexamprob}{\bar{\rnd{\chi}}}
\newcommand{\activeBatches}{\mathcal{A}}

\begin{document}

\icmltitlerunning{Online Learning to Rank in Stochastic Click Models}

\twocolumn[
\icmltitle{Online Learning to Rank in Stochastic Click Models}

\icmlsetsymbol{equal}{*}

\begin{icmlauthorlist}
\icmlauthor{Masrour Zoghi}{IR}
\icmlauthor{Tomas Tunys}{Cz}
\icmlauthor{Mohammad Ghavamzadeh}{DM}
\icmlauthor{Branislav Kveton}{AR}
\icmlauthor{Csaba Szepesvari}{UA}
\icmlauthor{Zheng Wen}{AR}
\end{icmlauthorlist}

\icmlaffiliation{IR}{Independent Researcher, Vancouver, BC, Canada (Part of this work was done during an internship at Adobe Research)}
\icmlaffiliation{Cz}{Czech Technical University, Prague, Czech Republic}
\icmlaffiliation{DM}{DeepMind, Mountain View, CA, USA (This work was done while the author was at Adobe Research)}
\icmlaffiliation{AR}{Adobe Research, San Jose, CA, USA}
\icmlaffiliation{UA}{University of Alberta, Edmonton, AB, Canada}

\icmlcorrespondingauthor{Branislav Kveton}{kveton@adobe.com}
\icmlcorrespondingauthor{Masrour Zoghi}{masrour@zoghi.org}

\vskip 0.3in]

\printAffiliationsAndNotice{}

\begin{abstract}
Online learning to rank is a core problem in information retrieval and machine learning. Many provably efficient algorithms have been recently proposed for this problem in specific click models. The click model is a model of how the user interacts with a list of documents. Though these results are significant, their impact on practice is limited, because all proposed algorithms are designed for specific click models and lack convergence guarantees in other models. In this work, we propose $\mergerank$, the first online learning to rank algorithm for a broad class of click models. The class encompasses two most fundamental click models, the cascade and position-based models. We derive a gap-dependent upper bound on the $T$-step regret of $\mergerank$ and evaluate it on a range of web search queries. We observe that $\mergerank$ outperforms ranked bandits and is more robust than $\cascadeklucb$, an existing algorithm for the cascade model.
\end{abstract}

\input{Introduction}

\input{Background}

\input{Model}

\input{Algorithm}

\input{Analysis}

\input{Experiments}

\input{RelatedWork}

\input{Conclusions}

\bibliographystyle{icml2017}
\bibliography{References}

\input{Notation}

\input{Appendix}

\end{document}

%% file: Introduction.tex

\section{Introduction}
\label{sec:introduction}

Learning to rank (LTR) is a core problem in information retrieval \cite{liu11learning} and machine learning; with numerous applications in web search, recommender systems and ad placement. The goal of LTR is to present a list of $K$ documents out of $L$ that maximizes the satisfaction of the user. This problem has been traditionally solved by training supervised learning models on manually annotated relevance judgments. However, strong evidence suggests \cite{agichtein06improving,zoghi16clickbased} that the feedback of users, that is clicks, can lead to major improvements over supervised LTR methods. In addition, billions of users interact daily with commercial LTR systems, and it is finally feasible to interactively and adaptive maximize the satisfaction of these users from clicks.

These observations motivated numerous papers on online LTR methods, which utilize user feedback to improve the quality of ranked lists. These methods can be divided into two groups: learning the best ranker in a family of rankers \cite{yue09interactively,hofmann13reusing}, and learning the best list under some model of user interaction with the list \cite{radlinski08learning,slivkins13ranked}, such as a \emph{click model} \cite{chuklin15click}. The click model is a stochastic model of how the user examines and clicks on a list of documents. In this work, we focus on online LTR in click models and address a shortcoming of all past work on this topic.

More precisely, many algorithms have been proposed and analyzed for finding the optimal ranked list in the cascade model (CM) \cite{kveton15cascading,combes15learning,kveton15combinatorial,zong16cascading,li16contextual}, the dependent-click model (DCM) \cite{katariya16dcm}, and the position-based model (PBM) \cite{lagree16multipleplay}. The problem is that if the user interacts with ranked lists using a different click model, the theoretical guarantees cease to hold. Then, as we show empirically, these algorithms may converge to suboptimal solutions. This is a grave issue because it is well known that no single click model captures the behavior of an entire population of users \cite{grotov15comparative}. Therefore, it is critical to develop efficient learning algorithms for multiple click models, which is the aim of this paper.

We make the following contributions:

\begin{itemize}[leftmargin=*,topsep=0mm,parsep=0pt,itemsep=0pt,partopsep=0pt]
\item We propose \emph{stochastic click bandits}, a learning framework for maximizing the expected number of clicks in online LTR in a broad class of click models, which includes both the PBM \cite{richardson07predicting} and CM \cite{craswell08experimental}.
\item We propose the first algorithm, $\mergerank$, that is guaranteed to learn the optimal solution in a diverse class of click models. This is of a great practical significance, as it is often difficult or impossible to guess the underlying click model in advance.
\item We prove a gap-dependent upper bound on the regret of $\mergerank$ that scales well with all quantities of interest. The key step in our analysis is a KL scaling lemma (\cref{sec:discussion}), which should be of a broader interest.
\item We evaluate $\mergerank$ on both CM and PBM queries. Our results show that $\mergerank$ performs significantly better than $\rankedexpthree$ \cite{radlinski08learning}, an adversarial online LTR algorithm; and is more robust than $\cascadeklucb$ \cite{kveton15cascading}, an optimal online LTR algorithm for the CM.
\end{itemize}

We define $[n] = \set{1, \dots, n}$. For any sets $A$ and $B$, we denote by $A^B$ the set of all vectors whose entries are indexed by $B$ and take values from $A$. We use boldface letters to denote important random variables.

%% file: Background.tex

\section{Background}
\label{sec:background}

This section reviews two fundamental \emph{click models} \cite{chuklin15click}, models of how users click on an ordered list of $K$ documents. The universe of all documents is represented by \emph{ground set} $\cD = [L]$ and we refer to the documents in $\cD$ as \emph{items}. The user is presented a \emph{ranked list}, an ordered list of $K$ documents out of $L$. We denote this list by $\cR = (d_1, \dots, d_K) \in \Pi_K(\cD)$, where $\Pi_K(\cD) \subset \cD^K$ is the set of all $K$-tuples with distinct elements from $\cD$ and $d_k$ is the $k$-th item in $\cR$. We assume that the click model is parameterized by $L$ \emph{item-dependent attraction probabilities} $\attrprob \in [0, 1]^L$, where $\attrprob(d)$ is the probability that item $d$ is \emph{attractive}. The items attract the user independently. For simplicity and without loss of generality, we assume that $\attrprob(1) \geq \ldots \geq \attrprob(L)$. The reviewed models differ in how the user examines items, which leads to clicks.

\subsection{Position-Based Model}
\label{sec:PBM}

The \emph{position-based model (PBM)} \cite{richardson07predicting} is a model where the probability of clicking on an item depends on both its identity and position. Therefore, in addition to item-dependent attraction probabilities, the PBM is parameterized by $K$ \emph{position-dependent examination probabilities} $\examprob \in [0, 1]^K$, where $\examprob(k)$ is the examination probability of position $k$.

The user interacts with a list of items $\cR = (d_1, \dots, d_K)$ as follows. The user \emph{examines} position $k \in [K]$ with probability $\examprob(k)$ and then \emph{clicks} on item $d_k$ at that position with probability $\attrprob(d_k)$. Thus, the expected number of clicks on list $\cR$ is
\begin{align*}
  r(\cR) = \sum_{k = 1}^K \examprob(k) \attrprob(d_k)\,.
\end{align*}
In practice, it is often observed that $\examprob(1) \geq \ldots \geq \examprob(K)$ \cite{chuklin15click}, and we adopt this assumption in this work. Under this assumption, the above function is maximized by the list of $K$ most attractive items 
\begin{align}
  \cR^\ast = (1, \dots, K)\,,
  \label{eq:optimal solution}
\end{align}
where the $k$-th most attractive item is placed at position $k$.

In this paper, we focus on the objective of maximizing the number of clicks. We note that the satisfaction of the user may not increase with the number of clicks, and that other objectives have been proposed in the literature \cite{radlinski08how}. The shortcoming of all of these objectives is that none directly measure the satisfaction of the user.

\subsection{Cascade Model}
\label{sec:CM}

In the \emph{cascade model (CM)} \cite{craswell08experimental}, the user scans a list of items $\cR = (d_1, \dots, d_K)$ from the first item $d_1$ to the last $d_K$. If item $d_k$ is \emph{attractive}, the user \emph{clicks} on it and does not examine the remaining items. If item $d_k$ is not attractive, the user \emph{examines} item $d_{k + 1}$. The first item $d_1$ is examined with probability one.

From the definition of the model, the probability that item $d_k$ is examined is equal to the probability that none of the first $k - 1$ items are attractive. Since each item attracts the user independently, this probability is
\begin{align}
  \examprob(\cR, k) = \prod_{i = 1}^{k - 1} (1 - \attrprob(d_i))\,.
  \label{eq:CM examination}
\end{align}
The expected number of clicks on list $\cR$ is at most $1$, and is equal to the probability of observing any click,
\begin{align*}
  r(\cR) =
  \sum_{k = 1}^K \examprob(\cR, k) \attrprob(d_k) =
  1 - \prod_{k = 1}^K (1 - \attrprob(d_k))\,.
\end{align*}
This function is maximized by the list of $K$ most attractive items $\cR^\ast$ in \eqref{eq:optimal solution}, though any permutation of $[K]$ would be optimal in the CM. Note that the list $\cR^\ast$ is optimal in both the PBM and CM.

%% file: Model.tex

\section{Online Learning to Rank in Click Models}
\label{sec:online learning to rank}

The PBM and CM (\cref{sec:background}) are similar in many aspects. First, both models are parameterized by $L$ item-dependent attraction probabilities. The items attract the user independently. Second, the probability of clicking on the item is a product of its attraction probability, which depends on the identity of the item; and the examination probability of its position, which is independent of the identity of the item. Finally, the optimal solution in both models is the list of $K$ most attractive items $\cR^\ast$ in \eqref{eq:optimal solution}, where the $k$-th most attractive item is placed at position $k$.

This suggests that the optimal solution in both models can be learned by a single learning algorithm, which does not know the underlying model. We propose this algorithm in \cref{sec:algorithm}. Before we discuss it, we formalize our learning problem as a \emph{multi-armed bandit} \cite{auer02finitetime,lai85asymptotically}.

\subsection{Stochastic Click Bandit}
\label{sec:stochastic click bandit}

We refer to our learning problem as a \emph{stochastic click bandit}. An instance of this problem is a tuple $(K, L, \allowbreak P_\attrprob, P_\examprob)$, where $K$ is the number of positions, $L$ is the number of items, $P_\attrprob$ is a distribution over binary vectors $\set{0, 1}^L$, and $P_\examprob$ is a distribution over binary matrices $\set{0, 1}^{\Pi_K(\cD) \times K}$.

The learning agent interacts with our problem as follows. Let $(\attrRV_t, \examRV_t)_{t = 1}^T$ be $T$ i.i.d. random variables drawn from $P_\attrprob \otimes P_\examprob$, where $\attrRV_t \in \set{0, 1}^L$ and $\attrRV_t(d)$ is the \emph{attraction indicator} of item $d$ at time $t$; and $\examRV_t \in \set{0, 1}^{\Pi_K(\cD) \times K}$ and $\examRV_t(\cR, k)$ is the \emph{examination indicator} of position $k$ in list $\cR \in \Pi_K(\cD)$ at time $t$. At time $t$, the agent chooses a list $\rnd{\cR}_t = (\rnd{d}^t_1, \dots, \rnd{d}^t_K)\in \Pi_K(\cD)$, which depends on past observations of the agent, and then observes \emph{clicks}. These clicks are a function of $\rnd{\cR}_t$, $\attrRV_t$, and $\examRV_t$. Let $\rnd{c}_t \in \set{0, 1}^K$ be the vector of \emph{click indicators} on all positions at time $t$. Then
\begin{align*}
  \rnd{c}_t(k) = \examRV_t(\rnd{\cR}_t, k) \attrRV_t(\rnd{d}^t_k)
\end{align*}
for any $k \in [K]$, the item at position $k$ is clicked only if both $\examRV_t(\rnd{\cR}_t, k) = 1$ and $\attrRV_t(\rnd{d}^t_k) = 1$.

The goal of the learning agent is to maximize the number of clicks. Therefore, the number of clicks at time $t$ is the \emph{reward} of the agent at time $t$. We define it as
\begin{align}
  \rnd{r}_t =
  \sum_{k = 1}^K \rnd{c}_t(k) =
  r(\rnd{\cR}_t, \attrRV_t, \examRV_t)\,,
  \label{eq:reward}
\end{align}
where $r: \Pi_K(\cD) \times [0, 1]^L \times [0, 1]^{\Pi_K(\cD) \times K} \to [0, K]$ is a \emph{reward function}, which we define for any $\cR \in \Pi_K(\cD)$, $\attrval \in [0, 1]^L$, and $\examval \in [0, 1]^{\Pi_K(\cD) \times K}$ as
\begin{align*}
  r(\cR, \attrval, \examval) = \sum_{k = 1}^K \examval(\cR, k) \attrval(d_k)\,.
\end{align*}

We adopt the same independence assumptions as in \cref{sec:background}. In particular, we assume that items attract the user independently.

\begin{assumption}
\label{ass:attraction independence} For any $\attrval \in \set{0, 1}^L$,
\begin{align*}
  \textstyle
  P(\attrRV_t = \attrval) = \prod_{d \in \cD} \mathrm{Ber}(\attrval(d); \attrprob(d))\,,
\end{align*}
where $\mathrm{Ber}(\cdot; \theta)$ denotes the probability mass function of a Bernoulli distribution with mean $\theta \in [0, 1]$, which we define as $\mathrm{Ber}(y; \theta) = \theta^y (1 - \theta)^{1 - y}$ for any $y \in \set{0, 1}$.
\end{assumption}

Moreover, we assume that the attraction of any item is independent of the examination of its position.

\begin{assumption}
\label{ass:attraction-examination independence} For any list $\cR \in \Pi_K(\cD)$ and position $k$,
\begin{align*}
  \condE{\rnd{c}_t(k)}{\rnd{\cR}_t = \cR} =
  \examprob(\cR, k) \attrprob(d_k)\,,
\end{align*}
where $\examprob \in [0, 1]^{\Pi_K(\cD) \times K}$ and $\examprob(\cR, k) = \E{\examRV_t(\cR, k)}$ is the \emph{examination probability} of position $k$ in list $\cR$.
\end{assumption}

We do not make any independence assumptions among the entries of $\examRV_t$, and on other interactions of $\attrRV_t$ and $\examRV_t$.

From our independence assumptions and the definition of the reward in \eqref{eq:reward}, the \emph{expected reward} of list $\cR$ is
\begin{align*}
  \E{r(\cR, \attrRV_t, \examRV_t)} =
  \sum_{k = 1}^K \examprob(\cR, k) \attrprob(d_k) =
  r(\cR, \attrprob, \examprob)\,.
\end{align*}
We evaluate the performance of a learning agent by its \emph{expected cumulative regret}
\begin{align*}
  R(T) = \E{\sum_{t = 1}^T R(\rnd{\cR}_t, \attrRV_t, \examRV_t)}\,,
\end{align*}
where $R(\rnd{\cR}_t, \attrRV_t, \examRV_t) = r(\cR^\ast, \attrRV_t, \examRV_t) - r(\rnd{\cR}_t, \attrRV_t, \examRV_t)$ is the \emph{instantaneous regret} of the agent at time $t$ and
\begin{align*}
  \textstyle
  \cR^\ast = \argmax_{\cR \in \Pi_K(\cD)} r(\cR, \attrprob, \examprob)
\end{align*}
is the \emph{optimal list} of items, the list that maximizes the expected reward. To simplify exposition, we assume that the optimal solution, as a set, is unique.

\subsection{Position Bandit}
\label{sec:position bandit}

The learning variant of the PBM in \cref{sec:PBM} can be formulated in our setting when
\begin{align}
  \forall \cR, \cR' \in \Pi_K(\cD): \ \examRV_t(\cR, k) = \examRV_t(\cR', k)
  \label{eq:CM examination indicator}
\end{align}
at any position $k \in [K]$. Under this assumption, the probability of clicking on item $\rnd{d}^t_k$ at time $t$ is
\begin{align*}
  \condE{\rnd{c}_t(k)}{\rnd{\cR}_t} =
  \examprob(k) \attrprob(\rnd{d}^t_k)\,,
\end{align*}
where $\examprob(k)$ is defined in \cref{sec:PBM}. The expected reward of list $\rnd{\cR}_t$ at time $t$ is
\begin{align*}
  \condE{\rnd{r}_t}{\rnd{\cR}_t} =
  \sum_{k = 1}^K \examprob(k) \attrprob(\rnd{d}^t_k)\,.
\end{align*}

\subsection{Cascading Bandit}
\label{sec:cascading bandit}

The learning variant of the CM in \cref{sec:CM} can be formulated in our setting when
\begin{align}
  \examRV_t(\cR, k) = \prod_{i = 1}^{k - 1} (1 - \attrRV_t(d_i))
  \label{eq:PBM examination indicator}
\end{align}
for any list $\cR\in \Pi_K(\cD)$ and position $k \in [K]$. Under this assumption, the probability of clicking on item $\rnd{d}^t_k$ at time $t$ is
\begin{align*}
  \condE{\rnd{c}_t(k)}{\rnd{\cR}_t} =
  \left[\prod_{i = 1}^{k - 1} (1 - \attrprob(\rnd{d}^t_i))\right] \attrprob(\rnd{d}^t_k)\,.
\end{align*}
The expected reward of list $\rnd{\cR}_t$ at time $t$ is
\begin{align*}
  \condE{\rnd{r}_t}{\rnd{\cR}_t} =
  \sum_{k = 1}^K \left[\prod_{i = 1}^{k - 1} (1 - \attrprob(\rnd{d}^t_i))\right] \attrprob(\rnd{d}^t_k)\,.
\end{align*}

\subsection{Additional Assumptions}
\label{sec:additional assumptions}

The above assumptions are not sufficient to guarantee that the optimal list $\cR^\ast$ in \eqref{eq:optimal solution} is learnable. Therefore, we make four additional assumptions, which are quite natural.

\begin{assumption}[Order-independent examination]
\label{ass:order-independent examination} For any lists $\cR \in \Pi_K(\cD)$ and $\cR' \in \Pi_K(\cD)$, and position $k \in [K]$ such that $d_k = d_k'$ and $\set{d_1,\dots, d_{k - 1}} = \set{d'_1, \dots, d'_{k - 1}}$, $\examRV_t(\cR, k) = \examRV_t(\cR', k)$.
\end{assumption}

The above assumption says that the examination indicator $\examRV_t(\cR, k)$ only depends on the identities of $d_1, \dots, d_{k - 1}$. Both the CM and PBM satisfy this assumption, which can be validated from \eqref{eq:CM examination indicator} and \eqref{eq:PBM examination indicator}.

\begin{assumption}[Decreasing examination]
\label{ass:decreasing examination} For any list $\cR \in \Pi_K(\cD)$ and positions $1 \leq i \leq j \leq K$, $\examprob(\cR, i) \geq \examprob(\cR, j)$.
\end{assumption}

The above assumption says that a lower position cannot be examined more than a higher position, in any list $\cR$. Both the CM and PBM satisfy this assumption.

\begin{assumption}[Correct examination scaling]
\label{ass:examination scaling} For any list $\cR \in \Pi_K(\cD)$ and positions $1 \leq i \leq j \leq K$, let $\attrprob(d_i) \leq \attrprob(d_j)$ and $\cR' \in \Pi_K(\cD)$ be the same list as $\cR$ except that $d_i$ and $d_j$ are exchanged. Then $\examprob(\cR, j) \geq \examprob(\cR', j)$.
\end{assumption}

The above assumption says that the examination probability of a position cannot increase if the item at that position is swapped for a less-attractive higher-ranked item, in any list $\cR$. Both the CM and PBM satisfy this assumption. In the CM, the inequality follows directly from the definition of examination in \eqref{eq:CM examination}. In the PBM, $\examprob(\cR, j) = \examprob(\cR', j)$.

\begin{assumption}[Optimal examination]
\label{ass:optimal examination} For any list $\cR \in \Pi_K(\cD)$ and position $k \in [K]$, $\examprob(\cR, k) \geq \examprob(\cR^\ast, k)$.
\end{assumption}

This assumption says that any position $k$ is least examined if the first $k - 1$ items are optimal. Both the CM and PBM satisfy this assumption. In the CM, the inequality follows from the definition of examination in \eqref{eq:CM examination}. In the PBM, we have that $\examprob(\cR, k) = \examprob(\cR^\ast, k)$.

%% file: Algorithm.tex

\section{Algorithm $\mergerank$}
\label{sec:algorithm}

The design of $\mergerank$ (\cref{alg:merge rank}) builds on two key ideas. First, we \emph{randomize} the placement of items to avoid biases due to the click model. Second, we \emph{divide and conquer}; recursively divide the batches of items into more and less attractive items. The result is a sorted list of $K$ items, where the $k$-th most attractive item is placed at position $k$.

\begin{algorithm}[t]
  \caption{$\mergerank$}
  \label{alg:merge rank}
  \begin{algorithmic}[1]
    \STATE // Initialization
    \FOR{$b = 1, \dots, 2 K$}
      \FOR{$\ell = 0, \dots, T - 1$}
        \FORALL{$d \in \cD$}
          \STATE $\rnd{c}_{b, \ell}(d) \gets 0, \ \rnd{n}_{b, \ell}(d) \gets 0$
        \ENDFOR
      \ENDFOR
    \ENDFOR \vspace{0.1in}
    \STATE $\activeBatches \gets \set{1}, \ b_{\max} \gets 1$ \\
    \STATE $\rnd{I}_1 \gets (1, K), \ \rnd{B}_{1, 0} \gets \cD,\ \ell_1 \gets 0$
    \FOR{$t = 1, \dots, T$}
      \FORALL{$b \in \activeBatches$}
        \STATE ${\tt DisplayBatch}(b, t)$
      \ENDFOR
      \FORALL{$b \in \activeBatches$}
        \STATE ${\tt CollectClicks}(b, t)$
      \ENDFOR
      \FORALL{$b \in \activeBatches$}
        \STATE ${\tt UpdateBatch}(b, t)$
      \ENDFOR
    \ENDFOR
  \end{algorithmic}
\end{algorithm}

$\mergerank$ explores items in batches, which are indexed by integers $b > 0$. A \emph{batch} $b$ is associated with the initial set of items $\rnd{B}_{b, 0} \subseteq \cD$ and a range of \emph{positions} $\rnd{I}_b \in [K]^2$, where $\rnd{I}_b(1)$ is the \emph{highest position} in batch $b$, $\rnd{I}_b(2)$ is the \emph{lowest position} in batch $b$, and $\len(b) = \rnd{I}_b(2) - \rnd{I}_b(1) + 1$ is \emph{number of positions} in batch $b$. The batch is explored in \emph{stages}, which we index by integers $\ell > 0$. The remaining items in stage $\ell$ of batch $b$ are $\rnd{B}_{b, \ell} \subseteq \rnd{B}_{b, 0}$. The lengths of the stages quadruple. More precisely, any item $d \in \rnd{B}_{b, \ell}$ in stage $\ell$ is explored $n_\ell$ times, where $n_\ell = \ceils{16 \tilde{\Delta}_\ell^{-2} \log T}$ and $\tilde{\Delta}_\ell = 2^{- \ell}$. The current stage of batch $b$ is $\ell_b$.

Method ${\tt DisplayBatch}$ (\cref{alg:display batch}) explores batches as follows. In stage $\ell$ of batch $b$, we randomly choose $\len(b)$ least observed items in $\rnd{B}_{b, \ell}$ and display them at randomly chosen positions in $\rnd{I}_b$. If the number of these items is less than $\len(b)$, we mix them with randomly chosen more observed items, which are not explored. This exploration has two notable properties. First, it is uniform in the sense that no item in $\rnd{B}_{b, \ell}$ is explored more than once than any other item in $\rnd{B}_{b, \ell}$. Second, any item in $\rnd{B}_{b, \ell}$ appears in any list over $\rnd{B}_{b, \ell}$ with that item with the same probability. This is critical to avoid biases due to click models.

Method ${\tt CollectClicks}$ (\cref{alg:collect clicks}) collects feedback. We denote the \emph{number of observations} of item $d$ in stage $\ell$ of batch $b$ by $\rnd{n}_{b, \ell}(d)$ and the \emph{number of clicks} on that item by $\rnd{c}_{b, \ell}(d)$. At the end of the stage, all items $d \in \rnd{B}_{b, \ell}$ are observed exactly $n_\ell$ times and we estimate the probability of clicking on item $d$ as
\begin{align}
  \hat{\rnd{c}}_{b, \ell}(d) = \rnd{c}_{b, \ell}(d) / n_\ell\,.
  \label{eq:click estimator}
\end{align}

Method ${\tt UpdateBatch}$ (\cref{alg:update batch}) updates batches and has three main parts. First, we compute $\klucb$ upper and lower confidence bounds \cite{garivier11klucb} for all items $d \in \rnd{B}_{b, \ell}$ (lines $5$--$6$),
\begin{align*}
  \rnd{U}_{b, \ell}(d)
  & \gets \argmax_{q \in [\hat{\rnd{c}}_{b, \ell}(d), \, 1]} 
  \set{n_\ell \, \kl{\hat{\rnd{c}}_{b, \ell}(d)}{q} \leq \delta_T}, \\
  \rnd{L}_{b, \ell}(d)
  & \gets \argmin_{q \in [0, \, \hat{\rnd{c}}_{b, \ell}(d)]} 
  \set{n_\ell \, \kl{\hat{\rnd{c}}_{b, \ell}(d)}{q} \leq \delta_T},
\end{align*}
where $\kl{p}{q}$ denotes the \emph{Kullback-Leibler divergence} between Bernoulli random variables with means $p$ and $q$, and $\delta_T = \log T + 3 \log \log T$. Then we test whether batch $b$ can be safely divided into $s$ more attractive items and the rest (lines $7$--$15$). If it can, we split the batch into two new batches (lines $21$--$27$). The first batch contains $s$ items and is over positions $\rnd{I}_b(1), \dots, \rnd{I}_b(1) + s - 1$; and the second batch contains the remaining items and is over the remaining positions. The stage indices of new batches are initialized to $0$. If the batch can be divided at multiple positions $s$, we choose the highest $s$. If the batch is not divided, we eliminate items that cannot be at position $\rnd{I}_b(2)$ or higher with a high probability (lines $17$--$19$).

\begin{algorithm}[t]
  \caption{${\tt DisplayBatch}$}
  \label{alg:display batch}
  \begin{algorithmic}[1]
    \STATE \textbf{Input:} batch index $b$, time $t$ \vspace{0.1in}
    \STATE $\ell \gets \ell_b, \ n_{\min} \gets \min_{d \in \rnd{B}_{b, \ell}} \rnd{n}_{b, \ell}(d)$
    \STATE Let $d_1, \dots, d_{\abs{\rnd{B}_{b, \ell}}}$ be a random permutation of items $\rnd{B}_{b, \ell}$ such that $\rnd{n}_{b, \ell}(d_1) \leq \ldots \leq \rnd{n}_{b, \ell}(d_{\abs{\rnd{B}_{b, \ell}}})$
    \STATE Let $\pi \in \Pi_{\len(b)}([\len(b)])$ be a random permutation of position assignments
    \FOR{$k = \rnd{I}_b(1), \dots, \rnd{I}_b(2)$}
      \STATE $\rnd{d}^t_k \gets d_{\pi(k - \rnd{I}_b(1) + 1)}$
    \ENDFOR
  \end{algorithmic}
\end{algorithm}

\begin{algorithm}[t]
  \caption{${\tt CollectClicks}$}
  \label{alg:collect clicks}
  \begin{algorithmic}[1]
    \STATE \textbf{Input:} batch index $b$, time $t$ \vspace{0.1in}
    \STATE $\ell \gets \ell_b, \ n_{\min} \gets \min_{d \in \rnd{B}_{b, \ell}} \rnd{n}_{b, \ell}(d)$
    \FOR{$k = \rnd{I}_b(1), \dots, \rnd{I}_b(2)$}
      \IF{$\rnd{n}_{b, \ell}(\rnd{d}^t_k) = n_{\min}$}
        \STATE $\rnd{c}_{b, \ell}(\rnd{d}^t_k) \gets \rnd{c}_{b, \ell}(\rnd{d}^t_k) + \rnd{c}_t(k)$
        \STATE $\rnd{n}_{b, \ell}(\rnd{d}^t_k) \gets \rnd{n}_{b, \ell}(\rnd{d}^t_k) + 1$
      \ENDIF
    \ENDFOR
  \end{algorithmic}
\end{algorithm}

\begin{algorithm}[t]
  \caption{${\tt UpdateBatch}$}
  \label{alg:update batch}
  \begin{algorithmic}[1]
    \STATE \textbf{Input:} batch index $b$, time $t$ \vspace{0.1in}
    \STATE // End-of-stage elimination
    \STATE $\ell \gets \ell_b$
    \IF{$\min_{d \in \rnd{B}_{b, \ell}} \rnd{n}_{b, \ell}(d) = n_\ell$}
      \FORALL{$d \in \rnd{B}_{b, \ell}$}
        \STATE Compute $\rnd{U}_{b, \ell}(d)$ and $\rnd{L}_{b, \ell}(d)$
      \ENDFOR
      \STATE Let $d_1, \dots, d_{\abs{\rnd{B}_{b, \ell}}}$ be any permutation of items $\rnd{B}_{b, \ell}$ such that $\rnd{L}_{b, \ell}(d_1) \geq \ldots \geq \rnd{L}_{b, \ell}(d_{\abs{\rnd{B}_{b, \ell}}})$
      \FOR{$k = 1, \dots, \len(b)$}
        \STATE $B^+_k \gets \set{d_1, \dots, d_k}$
        \STATE $B^-_k \gets \rnd{B}_{b, \ell} \setminus B^+_k$
      \ENDFOR \vspace{0.1in}
      \STATE // Find a split at the position with the highest index
      \STATE $s \gets 0$
      \FOR{$k = 1, \dots, \len(b) - 1$}
        \IF{$\rnd{L}_{b, \ell}(d_k) > \max_{d \in B^-_k} \rnd{U}_{b, \ell}(d)$}
          \STATE $s \gets k$
        \ENDIF
      \ENDFOR
      \IF{$(s = 0) \text{ and } (\abs{\rnd{B}_{b, \ell}} > \len(b))$}
        \STATE // Next elimination stage
        \STATE $\rnd{B}_{b, \ell + 1} \gets \set{d \in \rnd{B}_{b, \ell}: \rnd{U}_{b, \ell}(d) \geq \rnd{L}_{b, \ell}(d_{\len(b)})}$
        \STATE $\ell_b \gets \ell_b + 1$
      \ELSIF{$s > 0$}
        \STATE // Split
        \STATE $\activeBatches \gets \activeBatches \cup \set{b_{\max} + 1, b_{\max} + 2} \setminus \set{b}$
        \STATE $\rnd{I}_{b_{\max} + 1} \gets (\rnd{I}_b(1), \rnd{I}_b(1) + s - 1)$
        \STATE $\rnd{B}_{b_{\max} + 1, 0} \gets B^+_s, \ \ell_{b_{\max} + 1} \gets 0$
        \STATE $\rnd{I}_{b_{\max} + 2} \gets (\rnd{I}_b(1) + s, \rnd{I}_b(2))$
        \STATE $\rnd{B}_{b_{\max} + 2, 0} \gets B^-_s, \ \ell_{b_{\max} + 2} \gets 0$
        \STATE $b_{\max} \gets b_{\max} + 2$
      \ENDIF
    \ENDIF
  \end{algorithmic}
\end{algorithm}

The set of active batches is denoted by $\activeBatches$, and we explore and update these batches in parallel. The highest index of the latest added batch is $b_{\max}$. Note that $b_{\max} \leq 2 K$, because any batch with at least two items is split at a unique position into two batches. $\mergerank$ is initialized with a single batch over all positions and items (lines $6$--$7$).

Note that by the design of ${\tt UpdateBatch}$, the following invariants hold. First, the positions of active batches $\activeBatches$ are a partition of $[K]$ at any time $t$. Second, any batch contains at least as many items as is the number of the positions in that batch. Finally, when $\rnd{I}_b(2) < K$, the number of items in batch $b$ is equal to the number of its positions.

%% file: Analysis.tex

\section{Analysis}
\label{sec:analysis}

In this section, we state our regret bound for $\mergerank$. Before we do so, we discuss our estimator of clicks in \eqref{eq:click estimator}. In particular, we show that \eqref{eq:click estimator} is the attraction probability of item $d$ scaled by the average examination probability in stage $\ell$ of batch $b$. The examination scaling preserves the order of attraction probabilities, and therefore $\mergerank$ can operate on \eqref{eq:click estimator} in place of $\attrprob(d)$.

\begin{figure*}[t]
  \centering
  \includegraphics{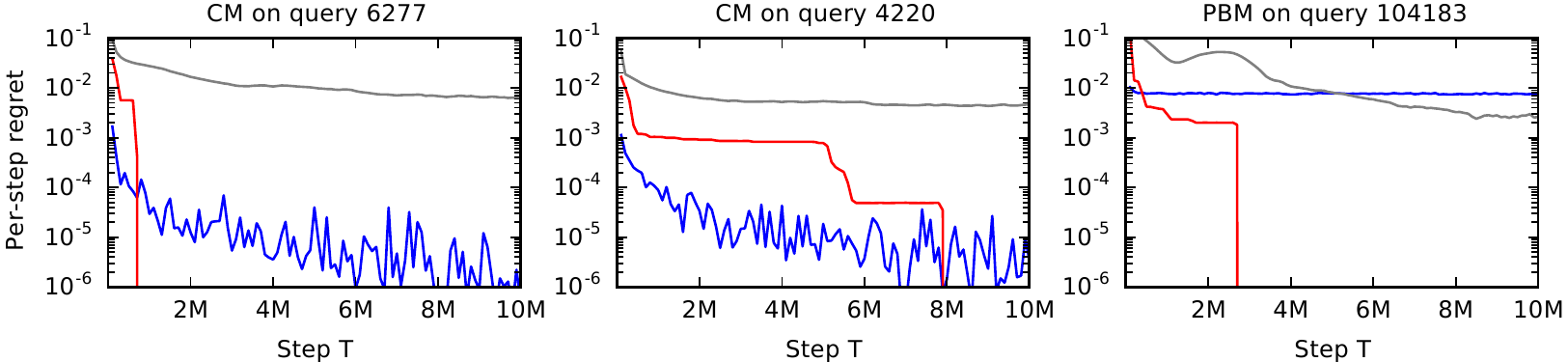}
  \vspace{-0.05in}
  \caption{The expected per-step regret of $\mergerank$ (red), $\cascadeklucb$ (blue), and $\rankedexpthree$ (gray) on three problems. The results are averaged over $10$ runs.}
  \label{fig:queries}
\end{figure*}

\subsection{Confidence Radii}
\label{sec:confidence radii}

Fix batch $b$, positions $\rnd{I}_b$, stage $\ell$, and items $\rnd{B}_{b, \ell}$. Then for any item $d \in \rnd{B}_{b,\ell}$, we can write the estimator in \eqref{eq:click estimator} as
\begin{align}
  \hat{\rnd{c}}_{b, \ell}(d) =
  \frac{1}{n_\ell} \sum_{t \in \mathcal{T}} \sum_{k = \rnd{I}_b(1)}^{\rnd{I}_b(2)} \rnd{c}_t(k) \I{\rnd{d}^t_k = d}
  \label{eq:click estimator 2}
\end{align}
for some set of $n_\ell$ time steps $\mathcal{T}$ and its expected value is
\begin{align}
  \bar{\rnd{c}}_{b, \ell}(d) = \E{\hat{\rnd{c}}_{b, \ell}(d)}\,.
  \label{eq:click probability}
\end{align}
The key step in the design of $\mergerank$ is that we maintain confidence radii around \eqref{eq:click estimator 2}. This is sound because the observations in \eqref{eq:click estimator 2},
\begin{align}
  \set{\examRV_t(\rnd{\cR}_t, k) \attrRV_t(d)}_{t \in \set{t \in \mathcal{T}: \, \rnd{d}^t_k = d}}
  \label{eq:observations}
\end{align}
at any position $k$, are i.i.d. in time. More precisely, by the design of ${\tt DisplayBatch}$, all displayed items from batch $b$ are chosen randomly from $\rnd{B}_{b, \ell}$; and independently of the realizations of $\examRV_t(\rnd{\cR}_t, k)$ and $\attrRV_t(d)$, which are random as well. The last problem is that the policy for placing items at positions $1, \dots, \rnd{I}_b(1) - 1$ can change independently of batch $b$ because $\mergerank$ splits batches independently. But this has no effect on $\examRV_t(\rnd{\cR}_t, k)$ because the examination of position $k$ does not depend on the order of higher ranked items (\cref{ass:order-independent examination}).

\subsection{Correct Examination Scaling}
\label{sec:correct examination scaling}

Fix batch $b$, positions $\rnd{I}_b$, stage $\ell$, and items $\rnd{B}_{b, \ell}$. Since the examination of a position does not depend on the order of higher ranked items, and does not depend on lower ranked items at all (\cref{ass:order-independent examination}), we can express the probability of clicking on any item $d \in \rnd{B}_{b, \ell}$ in \eqref{eq:click estimator 2} as
\begin{align}
  \bar{\rnd{c}}_{b, \ell}(d) =
  \frac{\attrprob(d)}{\abs{\mathcal{S}_d}} \sum_{\cR \in \mathcal{S}_d} \sum_{k = \rnd{I}_b(1)}^{\rnd{I}_b(2)}
  \examprob(\cR, k) \I{d_k = k}\,,
  \label{eq:click probability 2}
\end{align}
where
\begin{align*}
  \mathcal{S}_d =
  \big\{&(e_1, \dots, e_{\rnd{I}_b(2)}): d \in \set{e_{\rnd{I}_b(1)}, \dots, e_{\rnd{I}_b(2)}}\,, \\
  & (e_{\rnd{I}_b(1)}, \dots, e_{\rnd{I}_b(2)}) \in \Pi_{\len(b)}(\rnd{B}_{b, \ell})\big\}
\end{align*}
is the set of all lists with permutations of $\rnd{B}_{b, \ell}$ on positions $\rnd{I}_b$ that contain item $d$, for some fixed higher ranked items $e_1, \dots, e_{\rnd{I}_b(1) - 1} \notin \rnd{B}_{b, \ell}$. Let $d^\ast \in \rnd{B}_{b, \ell}$ be any item such that $\attrprob(d^\ast) \geq \attrprob(d)$, and $\bar{\rnd{c}}_{b, \ell}(d^\ast)$ and $\mathcal{S}_{d^\ast}$ be defined analogously to $\bar{\rnd{c}}_{b, \ell}(d)$ and $\mathcal{S}_d$ above. Then we argue that
\begin{align}
  \bar{\rnd{c}}_{b, \ell}(d^\ast) / \attrprob(d^\ast) \geq
  \bar{\rnd{c}}_{b, \ell}(d) / \attrprob(d)\,,
  \label{eq:correct examination scaling}
\end{align}
the examination scaling of a less attractive item $d$ is never higher than that of a more attractive item $d^\ast$.

Before we prove \eqref{eq:correct examination scaling}, note that for any list $\cR \in \mathcal{S}_d$, there exists one and only one list in $\mathcal{S}_{d^\ast}$ that differs from $\cR$ only in that items $d$ and $d^\ast$ are exchanged. Let this list be $\cR^\ast$. We analyze three cases. First, suppose that list $\cR$ does not contain item $d^\ast$. Then by \cref{ass:order-independent examination}, the examination probabilities of $d$ in $\cR$ and $d^\ast$ in $\cR^\ast$ are the same. Second, let item $d^\ast$ be ranked higher than item $d$ in $\cR$. Then by \cref{ass:examination scaling}, the examination probability of $d$ in $\cR$ is not higher than that of $d^\ast$ in $\cR^\ast$. Third, let item $d^\ast$ be ranked lower than item $d$ in $\cR$. Then by \cref{ass:order-independent examination}, the examination probabilities of $d$ in $\cR$ and $d^\ast$ in $\cR^\ast$ are the same, since they do not depend on lower ranked items. Finally, from the definition in \eqref{eq:click probability 2} and that $\abs{\mathcal{S}_d} = \abs{\mathcal{S}_{d^\ast}}$, we have that \eqref{eq:correct examination scaling} holds.

\subsection{Regret Bound}
\label{sec:regret bound}

For simplicity of exposition, let $\attrprob(1) > \ldots > \attrprob(L) > 0$. Let $\attrprob_{\max} = \attrprob(1)$, and $\examprob^\ast(k) = \examprob(\cR^\ast, k)$ for all $k \in [K]$. The regret of $\mergerank$ is bounded below.

\begin{theorem}
\label{thm:upper bound} For any stochastic click bandit in \cref{sec:stochastic click bandit} that satisfies Assumptions \ref{ass:attraction independence} to \ref{ass:optimal examination} and $T \geq 5$, the expected $T$-step regret of $\mergerank$ is bounded as
\begin{align*}
  R(T) \leq
  \frac{192 K^3 L}{(1 - \attrprob_{\max}) \Delta_{\min}} \log T + 4 K L (3 e + K)\,,
\end{align*}
where $\Delta_{\min} = \min_{k \in [K]} \set{\attrprob(k) - \attrprob(k + 1)}$.
\end{theorem}
\begin{proof}
The key idea is to bound the expected $T$-step regret in any batch (\cref{lem:batch regret} in Appendix). Since the number of batches is at most $2 K$, the regret of $\mergerank$ is at most $2 K$ times larger than that of in any batch.

The regret in a batch is bounded as follows. Let all confidence intervals hold, $\rnd{I}_b$ be the positions of batch $b$, and the maximum gap in batch $b$ be
\begin{align*}
  \textstyle
  \Delta_{\max} = \max_{d \in \set{\rnd{I}_b(1), \dots, \rnd{I}_b(2) - 1}} [\attrprob(d) - \attrprob(d + 1)]\,.
\end{align*}
If the gap of item $d$ in batch $b$ is $O(K \Delta_{\max})$, its regret is dominated by the time that the batch splits, and we bound this time in \cref{lem:split} in Appendix. Otherwise, the item is likely to be eliminated before the split, and we bound this time in \cref{lem:item elimination} in Appendix. Now take the maximum of these upper bounds.
\end{proof}

\begin{figure}[t]
  \centering
  \includegraphics{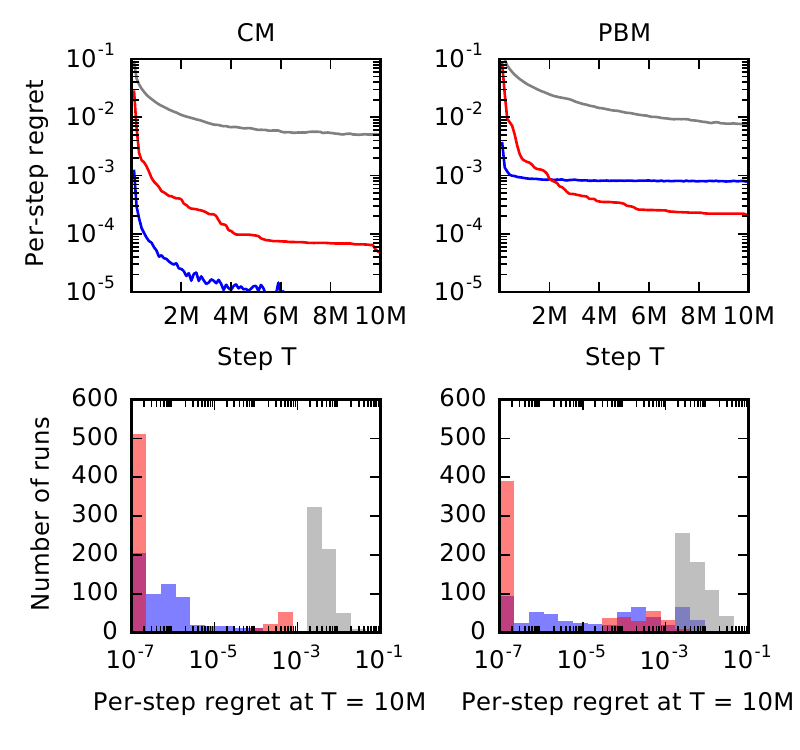}
  \vspace{-0.15in}
  \caption{The comparison of $\mergerank$ (red), $\cascadeklucb$ (blue), and $\rankedexpthree$ (gray) in the CM and PBM. In the top plots, we report the per-step regret as a function of time $T$, averaged over $60$ queries and $10$ runs per query. In the bottom plots, we show the distribution of the regret at $T = 10\text{M}$.}
  \label{fig:averages}
\end{figure}

\subsection{Discussion}
\label{sec:discussion}

Our upper bound in \cref{thm:upper bound} is logarithmic in the number of steps $T$, linear in the number of items $L$, and polynomial in the number of positions $K$. To the best of our knowledge, this is the first gap-dependent upper bound on the regret of a learning algorithm that has sublinear regret in both the CM and PBM. The gap $\Delta_{\min}$ characterizes the hardness of sorting $K + 1$ most attractive items, which is sufficient for solving our problem. In practice, the maximum attraction probability $\attrprob_{\max}$ is bounded away from $1$. Therefore, the dependence on $(1 - \attrprob_{\max})^{-1}$ is not critical. For instance, in most queries in \cref{sec:experiments}, $\attrprob_{\max} \leq 0.9$.

We believe that the cubic dependence on $K$ is not far from being optimal. In particular, consider the problem of learning the most clicked item-position pair in the PBM (\cref{sec:PBM}), which is easier than our problem. This problem can be solved as a stochastic rank-$1$ bandit \cite{katariya17stochastic} by $\rankoneelim$. Now consider the following PBM. The examination probability of the first position is close to one and the examination probabilities of all other positions are close to zero. Then the $T$-step regret of $\rankoneelim$ is $O([K^3 + K^2 L \Delta_{\min}^{-1}] \log T)$ because $\mu = O(1 / K)$, where $\mu$ is defined in \citet{katariya17stochastic}. Note that the gap-dependent term nearly matches our upper bound.

The KL confidence intervals in $\mergerank$ are necessary to achieve sample efficiency. The reason is, as we prove in \cref{lem:KL scaling} in Appendix, that
\begin{align*}
  \chi (1 - m) \kl{\alpha}{\alpha^\ast} \leq \kl{\chi \alpha}{\chi \alpha^\ast}
\end{align*}
for any $\chi, \alpha, \alpha^\ast \in [0, 1]$ and $m = \max \set{\alpha, \alpha^\ast}$. This implies that any two items with the expected rewards of $\chi \alpha$ and $\chi \alpha^\ast$ can be distinguished in $O(\chi^{-1} (\alpha^\ast - \alpha)^{-2})$ observations for any scaling factor $\chi$, when $m$ is bounded away from $1$. Suppose that $\alpha < \alpha^\ast$. Then the expected per-step regret for choosing the suboptimal item is $\chi (\alpha^\ast - \alpha)$, and the expected cumulative regret is $O((\alpha^\ast - \alpha)^{-1})$. The key observation is that the regret is independent of $\chi$. This is a major improvement over $\ucb$ confidence intervals, which only lead to $O(\chi^{-1} (\alpha^\ast - \alpha)^{-1})$ regret. Because $\chi$ can be exponentially small, such a dependence is undesirable.

The elimination of items in ${\tt UpdateBatch}$ (lines $17$--$19$) is necessary. The regret of $\mergerank$ would be quadratic in $L$ otherwise.

%% file: Experiments.tex

\section{Experiments}
\label{sec:experiments}

We experiment with the \emph{Yandex} dataset \cite{yandex}, a dataset of $35$ million (M) search sessions, each of which may contain multiple search queries. Each query is associated with displayed documents at positions $1$ to $10$ and their clicks. We select $60$ frequent search queries, and learn their CMs and PBMs using PyClick \cite{chuklin15click}, which is an open-source library of click models for web search. In each query, our goal it to rerank $L = 10$ most attractive items with the objective of maximizing the expected number of clicks at the first $K = 5$ positions. This resembles a real-world setting, where the learning agent would only be allowed to rerank highly attractive items, and not allowed to explore unattractive items \cite{zoghi16clickbased}.

$\mergerank$ is compared to two methods, $\cascadeklucb$ \cite{kveton15cascading} and $\rankedexpthree$ \cite{radlinski08learning}. $\cascadeklucb$ is an optimal algorithm for learning to rank in the cascade model. $\rankedexpthree$ is a variant of ranked bandits (\cref{sec:related work}) where the base bandit algorithm is $\expthree$ \cite{auer95gambling}. This approach is popular in practice and does not make any independence assumptions on the attractions of items.

Many solutions in our queries are near optimal, and therefore the optimal solutions are hard to learn. Therefore, we decided to evaluate the performance of algorithms by their \emph{expected per-step regret}, in up to $10\text{M}$ steps. If a solution is suboptimal and does not improve over time, its expected per-step regret remains constant and is bounded away from zero, and this can be easily observed even if the gap of the solution is small. We expect this when $\cascadeklucb$ is applied to the PBM because $\cascadeklucb$ has no guarantees in this model. The reported regret is averaged over periods of $100\text{k}$ steps to reduce randomness.

We report the performance of all compared algorithms on two CMs and one PBM in \cref{fig:queries}. The plots are chosen to represent general trends in this experiment. In the CM, $\cascadeklucb$ performs very well on most queries. This is not surprising since $\cascadeklucb$ is designed for the CM. $\mergerank$ often learns the optimal solution quickly (\cref{fig:queries}a), but sometimes this requires close to $T = 10\text{M}$ steps (\cref{fig:queries}b). In the PBM, $\cascadeklucb$ may converge to a suboptimal solution. Then its expected per-step regret remains constant and even $\rankedexpthree$ can learn a better solution over time (\cref{fig:queries}c). We also observe that $\mergerank$ outperforms $\rankedexpthree$ in all experiments.

We report the average performance of all compared algorithms in both click models in \cref{fig:averages}. These trends confirm our earlier findings. In the CM, $\cascadeklucb$ outperforms $\mergerank$; while in the PBM, $\mergerank$ outperforms $\cascadeklucb$ at $T = 2\text{M}$ steps. The regret of $\cascadeklucb$ in the PBM flattens and is bounded away from zero. This trend can be explained by the histograms in \cref{fig:averages}. They show that $\cascadeklucb$ converges to suboptimal solutions, whose regret is at least $10^{-3}$, in one sixth of its runs. The performance of $\mergerank$ is more robust and we do not observe many runs whose regret is of that magnitude.

We are delighted with the performance of $\mergerank$. Although it is not designed to be optimal (\cref{sec:discussion}), it is more robust than $\cascadeklucb$ and clearly outperforms $\rankedexpthree$. The performance of $\cascadeklucb$ is unexpectedly good. Although it does not have guarantees in the PBM, it performs very well on many queries. We plan to investigate this in our future work.

%% file: RelatedWork.tex

\section{Related Work}
\label{sec:related work}

A popular approach to online learning to rank are \emph{ranked bandits} \cite{radlinski08learning,slivkins13ranked}. The key idea in ranked bandits is to model each position in the recommended list as an individual bandit problem, which is then solved by a \emph{base bandit algorithm}. This algorithm is typically adversarial \cite{auer95gambling} because the distribution of clicks on lower positions is affected by higher positions. We compare to ranked bandits in \cref{sec:experiments}.

Online learning to rank in click models \cite{craswell08experimental,chuklin15click} was recently studied in several papers \cite{kveton15cascading,combes15learning,kveton15combinatorial,katariya16dcm,zong16cascading,li16contextual,lagree16multipleplay}. In all of these papers, the attraction probabilities of items are estimated from clicks and the click model. The learning agent has no guarantees beyond this model.

The problem of finding the most clicked item-position pair in the PBM, which is arguably easier than our problem of finding $K$ most clicked item-position pairs, can be solved as a stochastic rank-$1$ bandit \cite{katariya17stochastic,katariya17bernoulli}. We discuss our relation to these works in \cref{sec:discussion}.

Our problem can be also viewed as an instance of partial monitoring, where the attraction indicators of items are unobserved. General partial-monitoring algorithms \cite{agrawal89asymptotically,bartok12adaptive,bartok12partial,bartok14partial} are unsuitable for our setting because their computational complexity is polynomial in the number of actions, which is exponential in $K$.

The click model is a model of how the user interacts with a list of documents \cite{chuklin15click}, and many such models have been proposed \cite{becker07modeling,richardson07predicting,craswell08experimental,chapelle09dynamic,guo09click,guo09efficient}. Two fundamental click models are the CM \cite{craswell08experimental} and PBM \cite{richardson07predicting}. These models have been traditionally studied separately. In this work, we show that learning to rank problems in these models can be solved by the same algorithm, under reasonable assumptions.

%% file: Conclusions.tex

\section{Conclusions}
\label{sec:conclusions}

We propose stochastic click bandits, a framework for online learning to rank in a broad class of click models that encompasses two most fundamental click models, the cascade and position-based models. In addition, we propose a computationally and sample efficient algorithm for solving our problems, $\mergerank$, and derive an upper bound on its $T$-step regret. Finally, we evaluate $\mergerank$ on web search queries. Our algorithm outperforms ranked bandits \cite{radlinski08learning}, a popular online learning to rank approach; and is more robust than $\cascadeklucb$ \cite{kveton15cascading}, an existing algorithm for online learning to rank in the cascade model.

The goal of this work is not to propose the optimal algorithm for our setting, but to demonstrate that online learning to rank in multiple click models is possible with theoretical guarantees. We strongly believe that the design of $\mergerank$, as well as its analysis, can be improved. For instance, $\mergerank$ resets its estimators of clicks in each batch, which is wasteful. In addition, based on the discussion in \cref{sec:discussion}, our analysis may be loose by a factor of $K$. We hope that the practically relevant setting, which is introduced in this paper, will spawn new enthusiasm in the community and lead to more work in this area.

%% file: Notation.tex

\clearpage
\onecolumn
\appendix

\section{Notation}
\label{sec:notation}

{\renewcommand{\arraystretch}{1.2}
\begin{tabularx}{\textwidth}{|l|X|} \hline
  Symbol & Definition \\ \hline
  $\attrprob(d)$ & Attraction probability of item $d$ \\
  $\attrprob_{\max}$ & Highest attraction probability, $\attrprob(1)$ \\
  $\attrval$ & Binary attraction vector, where $\attrval(d)$ is the attraction indicator of item $d$ \\
  $P_\attrprob$ & Distribution over binary attraction vectors \\
  $\activeBatches$ & Set of active batches \\
  $b_{\max}$ & Index of the last created batch \\
  $\rnd{B}_{b, \ell}$ & Items in stage $\ell$ of batch $b$ \\
  $\rnd{c}_t(k)$ & Indicator of the click on position $k$ at time $t$ \\
  $\rnd{c}_{b, \ell}(d)$ & Number of observed clicks on item $d$ in stage $\ell$ of batch $b$ \\
  $\hat{\rnd{c}}_{b, \ell}(d)$ & Estimated probability of clicking on item $d$ in stage $\ell$ of batch $b$ \\
  $\bar{\rnd{c}}_{b, \ell}(d)$ & Probability of clicking on item $d$ in stage $\ell$ of batch $b$, $\E{\hat{\rnd{c}}_{b, \ell}(d)}$ \\
  $\cD$ & Ground set of items $[L]$ such that $\attrprob(1) \geq \ldots \geq \attrprob(L)$ \\
  $\delta_T$ & $\log T + 3 \log \log T$ \\
  $\tilde{\Delta}_\ell$ & $2^{- \ell}$ \\
  $\rnd{I}_b$ & Interval of positions in batch $b$ \\
  $K$ & Number of positions to display items \\
  $\len(b)$ & Number of positions to display items in batch $b$ \\
  $L$ & Number of items \\
  $\rnd{L}_{b, \ell}(d)$ & Lower confidence bound of item $d$, in stage $\ell$ of batch $b$ \\
  $n_\ell$ & Number of times that each item is observed in stage $\ell$ \\
  $\rnd{n}_{b, \ell}$ & Number of observations of item $d$ in stage $\ell$ of batch $b$ \\
  $\Pi_K(\cD)$ & Set of all $K$-tuples with distinct elements from $\cD$ \\
  $r(\cR, \attrval, \examval)$ & Reward of list $\cR$, for attraction and examination indicators $\attrval$ and $\examval$ \\
  $r(\cR, \attrprob, \examprob)$ & Expected reward of list $\cR$ \\
  $\cR = (d_1, \dots, d_K)$ & List of $K$ items, where $d_k$ is the $k$-th item in $\cR$ \\
  $\cR^\ast = (1, \dots, K)$ & Optimal list of $K$ items \\
  $R(\cR, \attrval, \examval)$ & Regret of list $\cR$, for attraction and examination indicators $\attrval$ and $\examval$ \\
  $R(T)$ & Expected cumulative regret in $T$ steps \\
  $T$ & Horizon of the experiment \\
  $\rnd{U}_{b, \ell}(d)$ & Upper confidence bound of item $d$, in stage $\ell$ of batch $b$ \\
  $\examprob(\cR, k)$ & Examination probability of position $k$ in list $\cR$ \\
  $\examprob^\ast(k)$ & Examination probability of position $k$ in the optimal list $\cR^\ast$ \\
  $\examval$ & Binary examination matrix, where $\examval(\cR, k)$ is the examination indicator of position $k$ in list $\cR$ \\
  $P_\examprob$ & Distribution over binary examination matrices \\ \hline
\end{tabularx}
}

%% file: Appendix.tex

\clearpage

\section{Proof of \cref{thm:upper bound}}
\label{sec:upper bound}

Let $\rnd{R}_{b, \ell}$ be the stochastic regret associated with stage $\ell$ of batch $b$. Then the expected $T$-step regret of ${\tt MergeRank}$ can be decomposed as
\begin{align*}
  R(T) \leq
  \E{\sum_{b = 1}^{2 K} \sum_{\ell = 0}^{T - 1} \rnd{R}_{b, \ell}}
\end{align*}
because the maximum number of batches is $2 K$. Let
\begin{align}
  \avgexamprob_{b, \ell}(d) = \frac{\bar{\rnd{c}}_{b, \ell}(d)}{\attrprob(d)}
  \label{eq:average examination probability}
\end{align}
be the \emph{average examination probability} of item $d$ in stage $\ell$ of batch $b$. Let
\begin{align*}
  \cE_{b, \ell} =
  \bigg\{\text{Event $1$: } & \forall d \in \rnd{B}_{b, \ell}:
  \bar{\rnd{c}}_{b, \ell}(d) \in [\rnd{L}_{b, \ell}(d), \rnd{U}_{b, \ell}(d)]\,, \\
  \text{Event $2$: } & \forall \rnd{I}_b \in [K]^2, \ d \in \rnd{B}_{b, \ell}, \ d^\ast \in \rnd{B}_{b, \ell} \cap [K]
  \text{ s.t. } \Delta = \attrprob(d^\ast) - \attrprob(d) > 0: \\
  & n_\ell \geq \frac{16 K}{\examprob^\ast(\rnd{I}_b(1)) (1 - \attrprob_{\max}) \Delta^2} \log T \implies
  \hat{\rnd{c}}_{b, \ell}(d) \leq \avgexamprob_{b, \ell}(d) [\attrprob(d) + \Delta / 4]\,, \\
  \text{Event $3$: } & \forall \rnd{I}_b \in [K]^2, \ d \in \rnd{B}_{b, \ell}, \ d^\ast \in \rnd{B}_{b, \ell} \cap [K]
  \text{ s.t. } \Delta = \attrprob(d^\ast) - \attrprob(d) > 0: \\
  & n_\ell \geq \frac{16 K}{\examprob^\ast(\rnd{I}_b(1)) (1 - \attrprob_{\max}) \Delta^2} \log T \implies
  \hat{\rnd{c}}_{b, \ell}(d^\ast) \geq \avgexamprob_{b, \ell}(d^\ast) [\attrprob(d^\ast) - \Delta / 4]\bigg\}
\end{align*}
be the ``good event'' in stage $\ell$ of batch $b$, where $\bar{\rnd{c}}_{b, \ell}(d)$ is the probability of clicking on item $d$ in stage $\ell$ of batch $b$, which is defined in \eqref{eq:click probability}; $\hat{\rnd{c}}_{b, \ell}(d)$ is its estimate, which is defined in \eqref{eq:click estimator 2}; and both $\examprob^\ast$ and $\attrprob_{\max}$ are defined in \cref{sec:regret bound}. Let $\overline{\cE_{b, \ell}}$ be the complement of event $\cE_{b, \ell}$. Let $\cE$ be the ``good event'' that all events $\cE_{b, \ell}$ happen; and $\ccE$ be its complement, the ``bad event'' that at least one event $\cE_{b, \ell}$ does not happen. Then the expected $T$-step regret can be bounded from above as
\begin{align*}
  R(T) \leq
  \E{\sum_{b = 1}^{2 K} \sum_{\ell = 0}^{T - 1} \rnd{R}_{b, \ell} \I{\cE}} + T P(\ccE) \leq
  \sum_{b = 1}^{2 K} \E{\sum_{\ell = 0}^{T - 1} \rnd{R}_{b, \ell} \I{\cE}} + 4 K L (3 e + K)\,,
\end{align*}
where the second inequality is from \cref{lem:bad events}. Now we apply \cref{lem:batch regret} to each batch $b$ and get that
\begin{align*}
  \sum_{b = 1}^{2 K} \E{\sum_{\ell = 0}^{T - 1} \rnd{R}_{b, \ell} \I{\cE}} \leq
  \frac{192 K^3 L}{(1 - \attrprob_{\max}) \Delta_{\min}} \log T\,.
\end{align*}
This concludes our proof.

\clearpage

\section{Upper Bound on the Probability of Bad Event $\ccE$}
\label{sec:bad event}

\begin{lemma}
\label{lem:bad events} Let $\ccE$ be defined as in the proof of Theorem \ref{thm:upper bound} and $T \geq 5$. Then
\begin{align*}
  P(\ccE) \leq
  \frac{4 K L (3 e + K)}{T}\,.
\end{align*}
\end{lemma}
\begin{proof}
By the union bound,
\begin{align*}
  P(\ccE) \leq
  \sum_{b = 1}^{2 K} \sum_{\ell = 0}^{T - 1} P(\overline{\cE_{b, \ell}})\,.
\end{align*}
Now we bound the probability of each event in $\overline{\cE_{b, \ell}}$ and then sum them up.

\vspace{0.1in}

\noindent \textbf{Event $1$}

\noindent The probability that event $1$ in $\cE_{b, \ell}$ does not happen is bounded as follows. Fix $\rnd{I}_b$ and $\rnd{B}_{b, \ell}$. For any $d \in \rnd{B}_{b, \ell}$,
\begin{align*}
  P(\bar{\rnd{c}}_{b, \ell}(d) \notin [\rnd{L}_{b, \ell}(d), \rnd{U}_{b, \ell}(d)])
  & \leq P(\bar{\rnd{c}}_{b, \ell}(d) < \rnd{L}_{b, \ell}(d)) + P(\bar{\rnd{c}}_{b, \ell}(d) > \rnd{U}_{b, \ell}(d)) \\
  & \leq \frac{2 e \ceils{\log(T \log^3 T) \log n_\ell}}{T \log^3 T} \\
  & \leq \frac{2 e \ceils{\log^2 T + \log(\log^3 T) \log T}}{T \log^3 T} \\
  & \leq \frac{2 e \ceils{2 \log^2 T}}{T \log^3 T} \\
  & \leq \frac{6 e}{T \log T}\,,
\end{align*}
where the second inequality is by Theorem 10 of \citet{garivier11klucb}, the third inequality is from $T \geq n_\ell$, the fourth inequality is from $\log(\log^3 T) \leq \log T$ for $T \geq 5$, and the last inequality is from $\ceils{2 \log^2 T} \leq 3 \log^2 T$ for $T \geq 3$. By the union bound,
\begin{align*}
  P(\exists d \in \rnd{B}_{b, \ell} \text{ s.t. } \bar{\rnd{c}}_{b, \ell}(d) \notin [\rnd{L}_{b, \ell}(d), \rnd{U}_{b, \ell}(d)]) \leq
  \frac{6 e L}{T \log T}
\end{align*}
for any $\rnd{B}_{b, \ell}$. Finally, since the above inequality holds for any $\rnd{B}_{b, \ell}$, the probability that event $1$ in $\cE_{b, \ell}$ does not happen is bounded as above.

\vspace{0.1in}

\noindent \textbf{Event $2$}

\noindent The probability that event $2$ in $\cE_{b, \ell}$ does not happen is bounded as follows. Fix $\rnd{I}_b$ and $\rnd{B}_{b, \ell}$, and let $k = \rnd{I}_b(1)$. If the event does not happen for items $d$ and $d^\ast$, then it must be true that
\begin{align*}
  n_\ell \geq \frac{16 K}{\examprob^\ast(k) (1 - \attrprob_{\max}) \Delta^2} \log T\,, \quad
  \hat{\rnd{c}}_{b, \ell}(d) > \avgexamprob_{b, \ell}(d) [\attrprob(d) + \Delta / 4]\,.
\end{align*}
From the definition of the average examination probability in \eqref{eq:average examination probability} and a variant of Hoeffding's inequality in \cref{lem:KL Hoeffding}, we have that
\begin{align*}
  P\left(\hat{\rnd{c}}_{b, \ell}(d) > \avgexamprob_{b, \ell}(d) [\attrprob(d) + \Delta / 4]\right) \leq
  \exp\left[- n_\ell \kl{\avgexamprob_{b, \ell}(d) [\attrprob(d) + \Delta / 4]}{\bar{\rnd{c}}_{b, \ell}(d)}\right]\,.
\end{align*}
From \cref{lem:KL scaling}, $\avgexamprob_{b, \ell}(d) \geq \examprob^\ast(k) / K$ (\cref{lem:examination}), and Pinsker's inequality, we have that
\begin{align*}
  \exp\left[- n_\ell \kl{\avgexamprob_{b, \ell}(d) [\attrprob(d) + \Delta / 4]}{\bar{\rnd{c}}_{b, \ell}(d)}\right]
  & \leq \exp\left[- n_\ell \avgexamprob_{b, \ell}(d) (1 - \attrprob_{\max})
  \kl{\attrprob(d) + \Delta / 4}{\attrprob(d)}\right] \\
  & \leq \exp\left[- n_\ell \frac{\examprob^\ast(k) (1 - \attrprob_{\max}) \Delta^2}{8 K}\right]\,.
\end{align*}
From our assumption on $n_\ell$, we conclude that
\begin{align*}
  \exp\left[- n_\ell \frac{\examprob^\ast(k) (1 - \attrprob_{\max}) \Delta^2}{8 K}\right] \leq
  \exp[- 2 \log T] =
  \frac{1}{T^2}\,.
\end{align*}
Finally, we chain all above inequalities and get that event $2$ in $\cE_{b, \ell}$ does not happen for any fixed $\rnd{I}_b$, $\rnd{B}_{b, \ell}$, $d$, and $d^\ast$ with probability of at most $T^{-2}$. Since the maximum numbers of items $d$ and $d^\ast$ are $L$ and $K$, respectively, the event does not happen for any fixed $\rnd{I}_b$ and $\rnd{B}_{b, \ell}$ with probability of at most $K L T^{- 2}$. In turn, the probability that event $2$ in $\cE_{b, \ell}$ does not happen is bounded by $K L T^{- 2}$.

\vspace{0.1in}

\noindent \textbf{Event $3$}

This bound is analogous to that of event $2$.

\vspace{0.1in}

\noindent \textbf{Total probability}

\noindent The maximum number of stages in any batch in $\mergerank$ is $\log T$ and the maximum number of batches is $2 K$. Hence, by the union bound,
\begin{align*}
  P(\ccE) \leq
  \left(\frac{6 e L}{T \log T} + \frac{K L}{T^2} + \frac{K L}{T^2}\right) (2 K \log T) \leq
  \frac{4 K L (3 e + K)}{T}\,.
\end{align*}
This concludes our proof.
\end{proof}

\clearpage

\section{Upper Bound on the Regret in Individual Batches}
\label{sec:batch regret}

\begin{lemma}
\label{lem:examination} For any batch $b$, positions $\rnd{I}_b$, stage $\ell$, set $\rnd{B}_{b, \ell}$, and item $d \in \rnd{B}_{b, \ell}$,
\begin{align*}
  \frac{\examprob^\ast(k)}{K} \leq
  \avgexamprob_{b, \ell}(d)\,,
\end{align*}
where $k = \rnd{I}_b(1)$ is the highest position in batch $b$.
\end{lemma}
\begin{proof}
The proof follows from two observations. First, by \cref{ass:optimal examination}, $\examprob^\ast(k)$ is the lowest examination probability of position $k$. Second, by the design of ${\tt DisplayBatch}$, item $d$ is placed at position $k$ with probability of at least $1 / K$.
\end{proof}

\begin{lemma}
\label{lem:separation} Let event $\cE$ happen and $T \geq 5$. For any batch $b$, positions $\rnd{I}_b$, set $\rnd{B}_{b, 0}$, item $d \in \rnd{B}_{b, 0}$, and item $d^\ast \in \rnd{B}_{b, 0} \cap \allowbreak [K]$ such that $\Delta = \allowbreak \attrprob(d^\ast) - \attrprob(d) > 0$, let $k = \rnd{I}_b(1)$ be the highest position in batch $b$ and $\ell$ be the first stage where
\begin{align*}
  \tilde{\Delta}_{\ell} < \sqrt{\frac{\examprob^\ast(k) (1 - \attrprob_{\max})}{K}} \Delta\,.
\end{align*}
Then $\rnd{U}_{b, \ell}(d) < \rnd{L}_{b, \ell}(d^\ast)$.
\end{lemma}
\begin{proof}
From the definition of $n_\ell$ in $\mergerank$ and our assumption on $\tilde{\Delta}_{\ell}$,
\begin{align}
  n_\ell \geq
  \frac{16}{\tilde{\Delta}_{\ell}^2} \log T >
  \frac{16 K}{\examprob^\ast(k) (1 - \attrprob_{\max}) \Delta^2} \log T\,.
  \label{eq:contradiction}
\end{align}
Let $\mu = \avgexamprob_{b, \ell}(d)$ and suppose that $\rnd{U}_{b, \ell}(d) \geq \mu [\attrprob(d) + \Delta / 2]$ holds. Then from this assumption, the definition of $\rnd{U}_{b, \ell}(d)$, and event $2$ in $\cE_{b, \ell}$,
\begin{align*}
  \kl{\hat{\rnd{c}}_{b, \ell}(d)}{\rnd{U}_{b, \ell}(d)}
  & \geq \kl{\hat{\rnd{c}}_{b, \ell}(d)}{\mu [\attrprob(d) + \Delta / 2]}
  \I{\hat{\rnd{c}}_{b, \ell}(d) \leq \mu [\attrprob(d) + \Delta / 2]} \\
  & \geq \kl{\mu [\attrprob(d) + \Delta / 4]}{\mu [\attrprob(d) + \Delta / 2]}\,.
\end{align*}
From \cref{lem:KL scaling}, $\mu \geq \examprob^\ast(k) / K$ (\cref{lem:examination}), and Pinsker's inequality, we have that
\begin{align*}
  \kl{\mu [\attrprob(d) + \Delta / 4]}{\mu [\attrprob(d) + \Delta / 2]}
  & \geq \mu (1 - \attrprob_{\max}) \kl{\attrprob(d) + \Delta / 4}{\attrprob(d) + \Delta / 2} \\
  & \geq \frac{\examprob^\ast(k) (1 - \attrprob_{\max}) \Delta^2}{8 K}\,.
\end{align*}
From the definition of $\rnd{U}_{b, \ell}(d)$, $T \geq 5$, and above inequalities,
\begin{align*}
  n_\ell =
  \frac{\log T + 3 \log \log T}{\kl{\hat{\rnd{c}}_{b, \ell}(d)}{\rnd{U}_{b, \ell}(d)}} \leq
  \frac{2 \log T}{\kl{\hat{\rnd{c}}_{b, \ell}(d)}{\rnd{U}_{b, \ell}(d)}} \leq
  \frac{16 K \log T}{\examprob^\ast(k) (1 - \attrprob_{\max}) \Delta^2}\,.
\end{align*}
This contradicts to \eqref{eq:contradiction}, and therefore it must be true that $\rnd{U}_{b, \ell}(d) < \mu [\attrprob(d) + \Delta / 2]$ holds.

On the other hand, let $\mu^\ast = \avgexamprob_{b, \ell}(d^\ast)$ and suppose that $\rnd{L}_{b, \ell}(d^\ast) \leq \mu^\ast [\attrprob(d^\ast) - \Delta / 2]$ holds. Then from this assumption, the definition of $\rnd{L}_{b, \ell}(d^\ast)$, and event $3$ in $\cE_{b, \ell}$,
\begin{align*}
  \kl{\hat{\rnd{c}}_{b, \ell}(d^\ast)}{\rnd{L}_{b, \ell}(d^\ast)}
  & \geq \kl{\hat{\rnd{c}}_{b, \ell}(d^\ast)}{\mu^\ast [\attrprob(d^\ast) - \Delta / 2]}
  \I{\hat{\rnd{c}}_{b, \ell}(d^\ast) \geq \mu^\ast [\attrprob(d^\ast) - \Delta / 2]} \\
  & \geq \kl{\mu^\ast [\attrprob(d^\ast) - \Delta / 4]}{\mu^\ast [\attrprob(d^\ast) - \Delta / 2]}\,.
\end{align*}
From \cref{lem:KL scaling}, $\mu^\ast \geq \examprob^\ast(k) / K$ (\cref{lem:examination}), and Pinsker's inequality, we have that
\begin{align*}
  \kl{\mu^\ast [\attrprob(d^\ast) - \Delta / 4]}{\mu^\ast [\attrprob(d^\ast) - \Delta / 2]}
  & \geq \mu^\ast (1 - \attrprob_{\max}) \kl{\attrprob(d^\ast) - \Delta / 4}{\attrprob(d^\ast) - \Delta / 2} \\
  & \geq \frac{\examprob^\ast(k) (1 - \attrprob_{\max}) \Delta^2}{8 K}\,.
\end{align*}
From the definition of $\rnd{L}_{b, \ell}(d^\ast)$, $T \geq 5$, and above inequalities,
\begin{align*}
  n_\ell =
  \frac{\log T + 3 \log \log T}{\kl{\hat{\rnd{c}}_{b, \ell}(d)}{\rnd{L}_{b, \ell}(d^\ast)}} \leq
  \frac{2 \log T}{\kl{\hat{\rnd{c}}_{b, \ell}(d^\ast)}{\rnd{L}_{b, \ell}(d^\ast)}} \leq
  \frac{16 K \log T}{\examprob^\ast(k) (1 - \attrprob_{\max}) \Delta^2}\,.
\end{align*}
This contradicts to \eqref{eq:contradiction}, and therefore it must be true that $\rnd{L}_{b, \ell}(d^\ast) > \mu^\ast [\attrprob(d^\ast) - \Delta / 2]$ holds.

Finally, based on inequality \eqref{eq:correct examination scaling},
\begin{align*}
  \mu^\ast =
  \frac{\bar{\rnd{c}}_{b, \ell}(d^\ast)}{\attrprob(d^\ast)} \geq
  \frac{\bar{\rnd{c}}_{b, \ell}(d)}{\attrprob(d)} =
  \mu\,,
\end{align*}
and item $d$ is guaranteed to be eliminated by the end of stage $\ell$ because
\begin{align*}
  \rnd{U}_{b, \ell}(d)
  & < \mu [\attrprob(d) + \Delta / 2] \\
  & \leq \mu \attrprob(d) + \frac{\mu^\ast \attrprob(d^\ast) - \mu \attrprob(d)}{2} \\
  & = \mu^\ast \attrprob(d^\ast) - \frac{\mu^\ast \attrprob(d^\ast) - \mu \attrprob(d)}{2} \\
  & \leq \mu^\ast [\attrprob(d^\ast) - \Delta / 2] \\
  & < \rnd{L}_{b, \ell}(d^\ast)\,.
\end{align*}
This concludes our proof.
\end{proof}

\begin{lemma}
\label{lem:item elimination} Let event $\cE$ happen and $T \geq 5$. For any batch $b$, positions $\rnd{I}_b$ where $\rnd{I}_b(2) = K$, set $\rnd{B}_{b, 0}$, and item $d \in \rnd{B}_{b, 0}$ such that $d > K$, let $k = \rnd{I}_b(1)$ be the highest position in batch $b$ and $\ell$ be the first stage where
\begin{align*}
  \tilde{\Delta}_{\ell} < \sqrt{\frac{\examprob^\ast(k) (1 - \attrprob_{\max})}{K}} \Delta
\end{align*}
for $\Delta = \attrprob(K) - \attrprob(d)$. Then item $d$ is eliminated by the end of stage $\ell$.
\end{lemma}
\begin{proof}
Let $B^+ = \set{k, \dots, K}$. Now note that $\attrprob(d^\ast) - \attrprob(d) \geq \Delta$ for any $d^\ast \in B^+$. By \cref{lem:separation}, $\rnd{L}_{b, \ell}(d^\ast) > \rnd{U}_{b, \ell}(d)$ for any $d^\ast \in B^+$; and therefore item $d$ is eliminated by the end of stage $\ell$.
\end{proof}

\begin{lemma}
\label{lem:split} Let $\cE$ happen and $T \geq 5$. For any batch $b$, positions $\rnd{I}_b$, and set $\rnd{B}_{b, 0}$, let $k = \rnd{I}_b(1)$ be the highest position in batch $b$ and $\ell$ be the first stage where
\begin{align*}
  \tilde{\Delta}_{\ell} < \sqrt{\frac{\examprob^\ast(k) (1 - \attrprob_{\max})}{K}} \Delta_{\max}
\end{align*}
for $\Delta_{\max} = \attrprob(s) - \attrprob(s + 1)$ and $s = \argmax\limits_{d \in \set{\rnd{I}_b(1), \dots, \rnd{I}_b(2) - 1}} [\attrprob(d) - \attrprob(d + 1)]$. Then batch $b$ is split by the end of stage $\ell$.
\end{lemma}
\begin{proof}
Let $B^+ = \set{k, \dots, s}$ and $B^- = \rnd{B}_{b, 0} \setminus B^+$. Now note that $\attrprob(d^\ast) - \attrprob(d) \geq \Delta_{\max}$ for any $(d^\ast, d) \in B^+ \times B^-$. By \cref{lem:separation}, $\rnd{L}_{b, \ell}(d^\ast) > \rnd{U}_{b, \ell}(d)$ for any $(d^\ast, d) \in B^+ \times B^-$; and therefore batch $b$ is split by the end of stage $\ell$.
\end{proof}

\begin{lemma}
\label{lem:batch regret} Let event $\cE$ happen and $T \geq 5$. Then the expected $T$-step regret in any batch $b$ is bounded as
\begin{align*}
  \E{\sum_{\ell = 0}^{T - 1} \rnd{R}_{b, \ell}} \leq
  \frac{96 K^2 L}{(1 - \attrprob_{\max}) \Delta_{\max}} \log T\,.
\end{align*}
\end{lemma}
\begin{proof}
Let $k = \rnd{I}_b(1)$ be the highest position in batch $b$. Choose any item $d \in \rnd{B}_{b, 0}$ and let $\Delta = \attrprob(k) - \attrprob(d)$.

First, we show that the expected per-step regret of any item $d$ is bounded by $\examprob^\ast(k) \Delta$ when event $\cE$ happens. Since event $\cE$ happens, all eliminations and splits up to any stage $\ell$ of batch $b$ are correct. Therefore, items $1, \dots, k - 1$ are at positions $1, \dots, k - 1$; and position $k$ is examined with probability $\examprob^\ast(k)$. Note that this is the highest examination probability in batch $b$ (\cref{ass:decreasing examination}). Our upper bound follows from the fact that the reward is linear in individual items (\cref{sec:stochastic click bandit}).

We analyze two cases. First, suppose that $\Delta \leq 2 K \Delta_{\max}$ for $\Delta_{\max}$ in \cref{lem:split}. Then by \cref{lem:split}, batch $b$ splits when the number of steps in a stage is at most
\begin{align*}
  \frac{16 K}{\examprob^\ast(k) (1 - \attrprob_{\max}) \Delta_{\max}^2} \log T\,.
\end{align*}
By the design of ${\tt DisplayBatch}$, any item in stage $\ell$ of batch $b$ is displayed at most $2 n_\ell$ times. Therefore, the maximum regret due to item $d$ in the last stage before the split is
\begin{align*}
  \frac{32 K \examprob^\ast(k) \Delta}{\examprob^\ast(k) (1 - \attrprob_{\max}) \Delta_{\max}^2} \log T \leq
  \frac{64 K^2 \Delta_{\max}}{(1 - \attrprob_{\max}) \Delta_{\max}^2} \log T =
  \frac{64 K^2}{(1 - \attrprob_{\max}) \Delta_{\max}} \log T\,.
\end{align*}
Now suppose that $\Delta > 2 K \Delta_{\max}$. This implies that item $d$ is easy to distinguish from item $K$. In particular,
\begin{align*}
  \attrprob(K) - \attrprob(d) =
  \Delta - (\attrprob(k) - \attrprob(K)) \geq
  \Delta - K \Delta_{\max} \geq
  \frac{\Delta}{2}\,,
\end{align*}
where the equality is from the identity
\begin{align*}
  \Delta =
  \attrprob(k) - \attrprob(d) =
  \attrprob(k) - \attrprob(K) + \attrprob(K) - \attrprob(d)\,;
\end{align*}
the first inequality is from $\attrprob(k) - \attrprob(K) \leq K \Delta_{\max}$, which follows from the definition of $\Delta_{\max}$ and $k \in [K]$; and the last inequality is from our assumption that $K \Delta_{\max} < \Delta / 2$. Now we apply the derived inequality and, by \cref{lem:item elimination} and from the design of ${\tt DisplayBatch}$, the maximum regret due to item $d$ in the stage where that item is eliminated is
\begin{align*}
  \frac{32 K \examprob^\ast(k) \Delta}{\examprob^\ast(k) (1 - \attrprob_{\max}) (\attrprob(K) - \attrprob(d))^2} \log T \leq
  \frac{128 K}{(1 - \attrprob_{\max}) \Delta} \log T \leq
  \frac{64}{(1 - \attrprob_{\max}) \Delta_{\max}} \log T\,.
\end{align*}
The last inequality is from our assumption that $\Delta > 2 K \Delta_{\max}$.

Because the lengths of the stages quadruple and $\mergerank$ resets all click estimators at the beginning of each stage, the maximum expected regret due to any item $d$ in batch $b$ is at most $1.5$ times higher than that in the last stage, and hence
\begin{align*}
  \E{\sum_{\ell = 0}^{T - 1} \rnd{R}_{b, \ell}} \leq
  \frac{96 K^2 \abs{\rnd{B}_{b, 0}}}{(1 - \attrprob_{\max}) \Delta_{\max}} \log T\,.
\end{align*}
This concludes our proof.
\end{proof}

\clearpage

\section{Technical Lemmas}
\label{sec:technical lemmas}

\begin{lemma}
\label{lem:KL Hoeffding} Let $(\rnd{X}_1)_{i = 1}^n$ be $n$ i.i.d. Bernoulli random variables, $\bar{\rnd{\mu}} = \sum_{i = 1}^n \rnd{X}_i$, and $\mu = \E{\bar{\rnd{\mu}}}$. Then
\begin{align*}
  P(\bar{\rnd{\mu}} \geq \mu + \varepsilon) \leq
  \exp[- n \kl{\mu + \varepsilon}{\mu}]
\end{align*}
for any $\varepsilon \in [0, 1 - \mu]$, and 
\begin{align*}
  P(\bar{\rnd{\mu}} \leq \mu - \varepsilon) \leq
  \exp[- n \kl{\mu - \varepsilon}{\mu}]
\end{align*}
for any $\varepsilon \in [0, \mu]$.
\end{lemma}
\begin{proof}
We only prove the first claim. The other claim follows from symmetry.

From inequality (2.1) of \citet{hoeffding63probability}, we have that
\begin{align*}
  P(\bar{\rnd{\mu}} \geq \mu + \varepsilon) \leq
  \left[\left(\frac{\mu}{\mu + \varepsilon}\right)^{\mu + \varepsilon}
  \left(\frac{1 - \mu}{1 - (\mu + \varepsilon)}\right)^{1 - (\mu + \varepsilon)}\right]^n
\end{align*}
for any $\varepsilon \in [0, 1 - \mu]$. Now note that
\begin{align*}
  \left[\left(\frac{\mu}{\mu + \varepsilon}\right)^{\mu + \varepsilon}
  \left(\frac{1 - \mu}{1 - (\mu + \varepsilon)}\right)^{1 - (\mu + \varepsilon)}\right]^n
  & = \exp\left[n \left[(\mu + \varepsilon) \log \frac{\mu}{\mu + \varepsilon} +
  (1 - (\mu + \varepsilon)) \log \frac{1 - \mu}{1 - (\mu + \varepsilon)}\right]\right] \\
  & = \exp\left[- n \left[(\mu + \varepsilon) \log \frac{\mu + \varepsilon}{\mu} +
  (1 - (\mu + \varepsilon)) \log \frac{1 - (\mu + \varepsilon)}{1 - \mu}\right]\right] \\
  & = \exp[- n \kl{\mu + \varepsilon}{\mu}]\,.
\end{align*}
This concludes the proof.
\end{proof}

\begin{lemma}
\label{lem:KL scaling} For any $c, p, q \in [0, 1]$,
\begin{align}
  c (1 - \max \set{p, q}) \kl{p}{q} \leq
  \kl{c p}{c q} \leq
  c \kl{p}{q}\,.
  \label{eq:KL scaling}
\end{align}
\end{lemma}
\begin{proof}
The proof is based on differentiation. The first two derivatives of $\kl{c p}{c q}$ with respect to $q$ are
\begin{align*}
  \frac{\partial}{\partial q} \kl{c p}{c q} = \frac{c (q - p)}{q (1 - c q)}\,, \quad
  \frac{\partial^2}{\partial q^2} \kl{c p}{c q} = \frac{c^2 (q - p)^2 + c p (1 - c p)}{q^2 (1 - c q)^2}\,;
\end{align*}
and the first two derivatives of $c \kl{p}{q}$ with respect to $q$ are
\begin{align*}
  \frac{\partial}{\partial q} [c \kl{p}{q}] = \frac{c (q - p)}{q (1 - q)}\,, \quad
  \frac{\partial^2}{\partial q^2} [c \kl{p}{q}] = \frac{c (q - p)^2 + c p (1 - p)}{q^2 (1 - q)^2}\,.
\end{align*}
The second derivatives show that $\kl{c p}{c q}$ and $c \kl{p}{q}$ are convex in $q$ for any $p$. Their minima are at $q = p$.

Now we fix $p$ and $c$, and prove \eqref{eq:KL scaling} for any $q$. The upper bound is derived as follows. Since
\begin{align*}
  \kl{c p}{c x} =
  c \kl{p}{x} =
  0
\end{align*}
when $x = p$, the upper bound holds when $c \kl{p}{x}$ increases faster than $\kl{c p}{c x}$ for any $p < x \leq q$, and when $c \kl{p}{x}$ decreases faster than $\kl{c p}{c x}$ for any $q \leq x < p$. This follows from the definitions of $\frac{\partial}{\partial x} \kl{c p}{c x}$ and $\frac{\partial}{\partial x} [c \kl{p}{x}]$. In particular, both derivatives have the same sign and $\abs{\frac{\partial}{\partial x} \kl{c p}{c x}} \leq \abs{\frac{\partial}{\partial x} [c \kl{p}{x}]}$ for any feasible $x \in [\min \set{p, q}, \max \set{p, q}]$.

The lower bound is derived as follows. The ratio of $\frac{\partial}{\partial x} [c \kl{p}{x}]$ and $\frac{\partial}{\partial x} \kl{c p}{c x}$ is bounded from above as
\begin{align*}
  \frac{\frac{\partial}{\partial x} [c \kl{p}{x}]}{\frac{\partial}{\partial x} \kl{c p}{c x}} =
  \frac{1 - c x}{1 - x} \leq
  \frac{1}{1 - x} \leq
  \frac{1}{1 - \max \set{p, q}}
\end{align*}
for any $x \in [\min \set{p, q}, \max \set{p, q}]$. Therefore, we get a lower bound on $\kl{c p}{c x}$ when we multiply $c \kl{p}{x}$ by $1 - \max \set{p, q}$.
\end{proof}